\documentclass[letterpaper, 10 pt, conference]{ieeeconf}  

\IEEEoverridecommandlockouts                              
\usepackage{times}

\usepackage{multicol}
\usepackage[bookmarks=true]{hyperref}
\usepackage{graphicx, mathtools}

\usepackage{amsmath,amssymb,amsfonts}
\usepackage{algorithm}
\usepackage{algpseudocode}

\usepackage{array}
\usepackage{newtxtext,newtxmath}
\usepackage[caption=false,font=footnotesize]{subfig}
\usepackage{textcomp}
\usepackage{stfloats}
\usepackage{url}
\usepackage{verbatim}
\usepackage{graphicx}
\usepackage{mathtools}
\usepackage{float}
\usepackage{xcolor}
\usepackage{makecell}
\usepackage{booktabs}  
\usepackage{array, caption, capt-of, cuted}     
\usepackage[font=small,labelfont=bf]{caption}

\let\labelindent\relax
\usepackage{enumitem}

\usepackage{cancel}

\usepackage{svg}
\usepackage{tabularx}

\usepackage{multirow}
\usepackage{makecell}
\usepackage{siunitx}
\usepackage{xcolor}

\usepackage{wrapfig}

\usepackage{eso-pic}

\newtheorem{theorem}{Theorem}
\newtheorem{lemma}{Lemma}
\newtheorem{prop}{Proposition}
\newtheorem{remark}{Remark}
\newtheorem{assumption}{Assumption}

\title{\LARGE \bf
Stein Variational Uncertainty-Adaptive Model Predictive Control
}

\author{Hrishikesh Sathyanarayan and Ian Abraham
\thanks{H. Sathyanarayan is with Department of Mechanical Engineering, Yale University, USA
        {\tt\small hrishi.sathyanarayan@yale.edu}}%
\thanks{I. Abraham is with the Department of Electrical Engineering, University of Sydney, Australia
        {\tt\small ian.abraham@sydney.edu.au}}%
}

\begin{document}

\AddToShipoutPictureFG*{%
  \AtPageUpperLeft{%
    \hspace*{0.75in}%
    \raisebox{-0.5in}{%
      {\fontfamily{phv}\fontsize{10}{12}\selectfont\bfseries
      \begin{tabular}{@{}l@{}}
      Conference on Decision and Control 2026\\
      Honolulu, Hawaii, USA, December 15--December 18, 2026
      \end{tabular}}%
    }%
  }%
}

\maketitle
\thispagestyle{empty}
\pagestyle{empty}

\begin{abstract}

We propose a Stein variational distributionally robust controller for nonlinear dynamical systems with latent parametric uncertainty. 
The method is an alternative to conservative worst-case ambiguity-set optimization with a deterministic particle-based approximation of a task-dependent uncertainty distribution, enabling the controller to concentrate on parameter sensitivities that most strongly affect closed-loop performance. 
Our method yields a controller that is robust to latent parameter uncertainty by coupling optimal control with Stein variational inference, and avoiding restrictive parametric assumptions on the uncertainty model while preserving computational parallelism. 
In contrast to classical DRO, which can sacrifice nominal performance through worst-case design, we find our approach achieves robustness by shaping the control law around relevant uncertainties that are most critical to the task objective. 
The proposed framework therefore reconciles robust control and variational inference in a single decision-theoretic formulation for broad classes of control systems with parameter uncertainty. 
We demonstrate our approach on representative control problems that empirically illustrate improved performance-robustness tradeoffs over nominal, ensemble, and classical distributionally robust baselines.

\end{abstract}

\section{INTRODUCTION}

Control of dynamical systems under broad 
uncertainty remains a central challenge in modern control theory.
In many applications, uncertainty arises from latent parameters in the system dynamics (i.e. mass, inertia, or geometry) that cannot be directly measured yet critically influence closed-loop performance.
Classical robust control methods address this challenge by optimizing performance under worst-case parameter uncertainty, yielding performance guarantees at the expense of conservatively robust controllers \cite{kuhn2025distributionallyrobustoptimization,long2025sensorbaseddistributionallyrobustcontrol,liu2025datadrivendistributionallyrobustoptimal}.
In contrast, stochastic and risk-sensitive control methods optimize expected performance subject to a prescribed distribution, but rely heavily on task-agnostic sampling methods that fail to accurately model task-relevant uncertainty in practice \cite{Abraham_2020,lowrey2019planonlinelearnoffline}.
Bridging this gap between robustness and control performance guarantees remains a fundamental open problem.

Distributionally Robust Control (DRO) \cite{kuhn2025distributionallyrobustoptimization} provides a principled approach for reasoning about uncertainty by optimizing over a set of admissible probability distributions, typically defined through divergence-based ambiguity sets. 
While DRO provides theoretically-grounded control formulations to model uncertainty, the method relies heavily on restrictive assumptions: either the uncertainty distribution must be parametrized and estimated a-priori \cite{long2025sensorbaseddistributionallyrobustcontrol,liu2025datadrivendistributionallyrobustoptimal}, or samples must be drawn from a predefined generative process.
Moreover, classical DRO inherits a worst-case uncertainty design formulation that leads to overly conservative controls that subsequently degrade task performance.

\begin{figure}
    \centering
    \includegraphics[width=\linewidth]{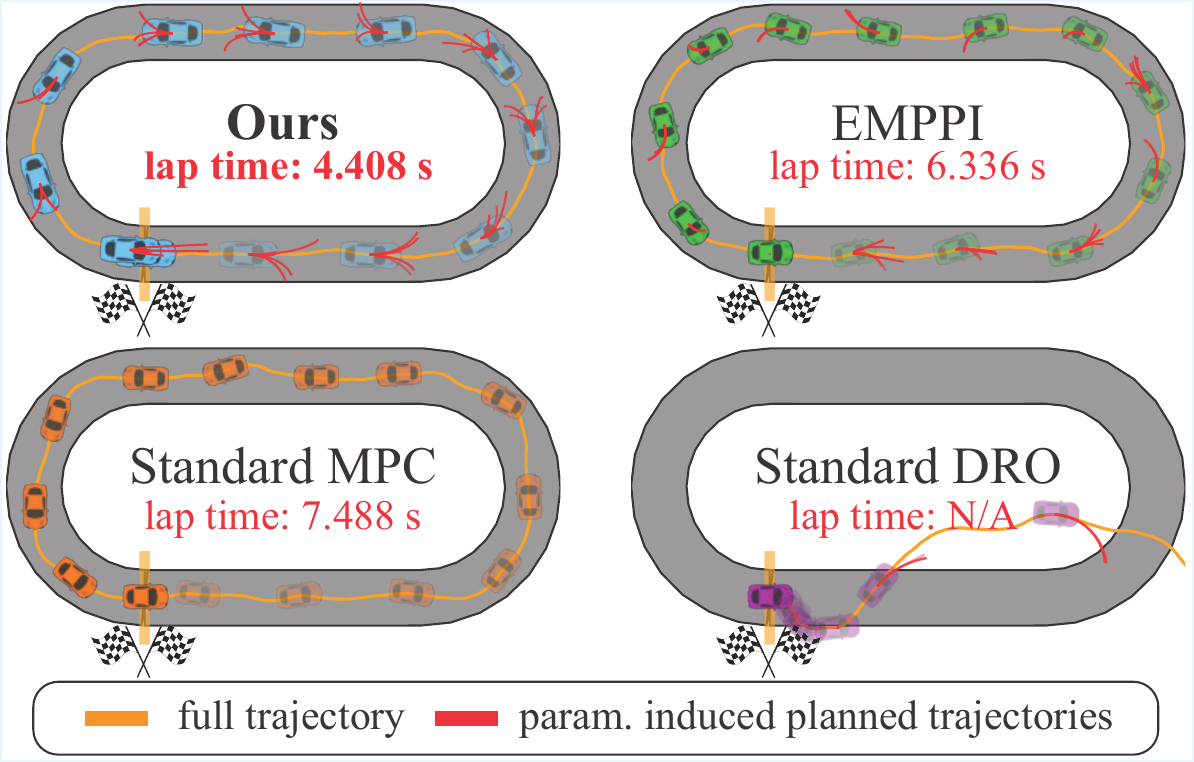}
    \caption{\textbf{Autonomous Racing under Vehicle Inertial Uncertainty. } 
    Here, we apply our control framework to autonomous racing, minimizing lap time under significant uncertainty in the vehicle’s mass distribution (mass and inertia).
    We compare our approach to state-of-the-art baselines: Ensemble Model Predictive Control \cite{Abraham_2020}, classical Distributionally Robust Control \cite{kuhn2025distributionallyrobustoptimization}, and nominal Model Predictive Control. 
    Our approach quickly adapts to \emph{task-sensitive} uncertainties faster than nominal baselines, enabling fast adaptation and convergence to the objective.}
    \label{fig:abstract}
    \vspace{-1.9em}
\end{figure}

A guiding question is how we can relax the assumptions of DRO to avoid worst-case uncertainty modeling.
Variational Inference (VI) methods have emerged as powerful tools that approximate complex probability distributions without committing to restrictive parametric families.
Stein Variational Gradient Descent (SVGD) \cite{liu2019steinvariationalgradientdescent}, in particular, constructs a deterministic flow of particles that approximates a target distribution via functional gradient descent in a reproducing kernel Hilbert space.
Unlike traditional VI, SVGD provides a nonparametric and computationally tractable mechanism to represent multimodal and task-dependent uncertainty.
Despite its success in Bayesian inference and learning \cite{zhuo2018messagepassingsteinvariational,korba2021nonasymptoticanalysissteinvariational}, its role in robust control synthesis remains largely unexplored.

Early work presented in \cite{barcelos2021dualonlinesteinvariational}  introduced Stein Variational Inference (SVI) \cite{lee2024steinvariationalergodicsearch,lambert2021steinvariationalmodelpredictive} for jointly reasoning over controls and dynamics through brute-force parallelization, but fails to capture task-sensitivities that could be used to reduce computation. 
This work shows that Stein variational inference over task-sensitive parameters alone is sufficient to achieve distributional robustness within a model-predictive control paradigm under limited sensor feedback, thereby avoiding parallel inference over controls while yielding reliable, uncertainty-robust control synthesis.

The core idea is to replace static or worst-case uncertainty models with a \emph{deterministic}, evolving set of particles that adapts to uncertainty via a \emph{task-dependent} posterior over parameter uncertainty.
Rather than optimizing over all admissible uncertainties, our method prioritizes uncertainty \emph{most sensitive to task performance}, concentrating computation on uncertainties that most impact control.
This approach leads to 
optimal control synthesis that is coupled and co-evolved with uncertainty propagation through a coupled, adversarial-based optimization. \footnote{This is distinctly different from belief-space planning \cite{platt2017efficient} and experimental design related methods \cite{sathyanarayan2025behaviorsynthesiscontactawarefisher,wilson_fishermax,vantilborgh2025dualcontrolreferencegeneration} in that we never explicitly reduce uncertainty, rather just focus on task performance.}
Furthermore, we characterize how the induced controller is shaped by adapting over task-relevant uncertainty.


In summary, we present a recast of the robust control problem as an inference-driven process where control synthesis is actively shaped to improve decision making subject to task-sensitive uncertainties. 
The main contributions of this work are as follows:

\begin{enumerate}
    \item A non-parametric, deterministic approximation of task-dependent uncertainty in place of worst-case, conservative uncertainty design.
    \item Derivation of theoretical guarantees for the existence, optimality, and convergence of the Stein variational approximation to the task-sensitive parameter distribution.
    \item Demonstrations of improved performance/robustness tradeoffs of the proposed approach compared to classical and stochastic sampling methods.
\end{enumerate}

\section{Problem Formulation}

    Let $x_t \in \mathcal{X}$ be the system state, $u_t \in \mathcal{U}$ be the control input $\forall t \in [0,t_h]$ where $t_h$ is the discrete planning horizon, and $\theta \in \Theta$ be the system latent parameters for a nonlinear system of evolving dynamics $x_{t+1} = f(x_t,u_t,\theta)$.
    The objective is to solve the following control problem,
    \vspace{-1.0em}
    \begin{equation}\label{eq:task_opt}
        \begin{aligned}
            \min_{x_{0:t_h},u_{0:t_h}} \quad 
            & \mathcal{J}(x_{0:t_h},u_{0:t_h},\theta) 
            = m(x_{t_h},u_{t_h},\theta) + \sum_{t=0}^{t_h-1} \ell(x_t,u_t,\theta) \\
            \text{s.t.} \quad
            & x_0 \in \mathcal{X}_0,\;
              u_t \in \mathcal{U},\;
              x_{t+1} = f(x_t,u_t,\theta)
        \end{aligned}
    \end{equation}
    where $x_0$ are the system initial conditions, $\ell(\cdot)$ is the stage cost, and $m(\cdot)$ is the terminal cost.
    The above optimization is iteratively solved for a single planning cycle $t \in [0,t_h]$ and the optimal first control input is administered to the system in a receding-horizon manner.
    \begin{assumption}[Set Compactness and Measurability]\label{assum:set_assumptions}
    Let $n_x,n_u,n_\theta \in \mathbb{N}$. The state, input, and parameter spaces satisfy
    \begin{equation}
        \mathcal{X} \in \mathcal{B}(\mathbb{R}^{n_x}), 
        \qquad
        \mathcal{U} \in \mathcal{B}(\mathbb{R}^{n_u}),
        \qquad
        \Theta \in \mathcal{B}(\mathbb{R}^{n_\theta}),
    \end{equation}
    respectively where $\mathcal{B}(\mathbb{R}^n)$ denotes the Borel $\sigma$-algebra on $\mathbb{R}^n$. Moreover, $\Theta$ is nonempty and compact.
    \end{assumption}
    \begin{remark}
        The uncertainty enters the control problem through both the dynamics and the performance index. Hence, for fixed $u_{0:t_h-1}$, the realized trajectory and realized cost are both functions of $\theta$.
    \end{remark}
    \begin{prop}[Existence of optimal state-control trajectory]\label{prop:existence}
        Given a fixed $\theta \in \Theta$, let the set of feasible trajectories be
        \begin{equation}
            \mathcal{T} := \left\{
            (x_{0:t_h},u_{0:t_h}) \;\middle|\;
            \begin{aligned}
            &x_0 \in \mathcal{X}_0, \\
            &u_t \in \mathcal{U}, \\
            &x_{t+1} = f(x_t,u_t,\theta), \quad \forall t\in[0,t_h-1]
            \end{aligned}
            \right\}.
        \end{equation}
        Suppose that:
        \begin{enumerate}
            \item $\mathcal{X}_0$ and $\mathcal{U}$ are compact,
            \item $f(\cdot,\cdot,\theta)$ is continuous in $(x,u)$,
            \item $\ell(\cdot,\cdot,\theta)$ and $m(\cdot,\cdot,\theta)$ are lower semicontinuous in $(x,u)$.
        \end{enumerate}
        Then the feasible set $\mathcal{T}$ is nonempty and compact, and $
        \exists\tau^* \in \mathcal{T}$ such that
        \begin{equation}
            \mathcal{J}(\tau^*,\theta) = \min_{\tau \in \mathcal{T}} \mathcal{J}(\tau,\theta).
        \end{equation}
    \end{prop}
        
        
    \begin{proof}
        Since $f(\cdot,\cdot,\theta)$ is continuous, induction over $t$ shows that the trajectory $x_{0:t_h}$ is a continuous function of $(x_0, u_{0:t_h-1})$.
        Hence the nonempty $\mathcal{T}$ is the image of the compact set $\mathcal{X}_0 \times \mathcal{U}^{t_h}$ under a continuous map. 
        Since finite sums preserve lower semicontinuity, $\mathcal{J}(\cdot,\theta)$ is lower semicontinuous on $\mathcal{T}$, and the Weierstrass extreme value theorem \cite{c48b12a01e96440c8e357b319a0a6180} yields a minimizer $\tau^* \in \mathcal{T}$ of $\mathcal{J}(\cdot,\theta)$.
    \end{proof}
    Note that we do not assume $\theta$ to be known when the control is synthesized.
    Instead, $\theta$ is modeled as an unknown element of $\Theta$, or more generally as a random variable on $\big(\Theta, \mathcal{B}(\Theta)\big)$.
    When prior information is available by a prior distribution $p_0(\theta)$, it is encoded by a reference measure $p \in \mathcal{P}(\Theta)$, where
    \begin{equation}
        \theta \sim p_0(\theta).
    \end{equation}
    For any admissible $(x_{0:t_h},u_{0:t_h})$, the map $\theta \mapsto \mathcal{J}(x_{0:t_h},u_{0:t_h},\theta)$
    is a random cost under any measure $p \in \mathcal{P}(\Theta)$.
    Thus, uncertainty about $\theta$ induces uncertainty about closed loop performance.
    
    \subsection{Uncertainty-Aware Optimal Control}
    
        This paper optimizes a control sequence $u_{0:T}$ with respect to the \emph{variability} of $\theta \mapsto \mathcal{J}(\cdot,\theta)$ over $\Theta$, leading to functionals over measures $p \in \mathcal{P}(\Theta)$ and a class of uncertainty-aware optimal control problems of the form,
        \begin{equation}\label{eq:minmax_original}
            \begin{split}
                \min_{x_{0:t_h},u_{0:t_h}} & \sup_{p\in \mathcal{A}} \mathbb{E}_p \big[ \mathcal{J}(x_{0:t_h},u_{0:t_h},\theta) \big] \\
                \textrm{s.t. } &
                \textrm{constraints in \eqref{eq:task_opt}}
            \end{split}
        \end{equation}
        where $\mathcal{A}\subseteq \mathcal{P}(\Theta)$ denotes a selected family of admissible distributions over the latent parameters.
        The formulation above highlights that robustness is defined with respect to distributions that capture parameter uncertainty in an adversarial manner, where the objective is to minimize an upper bound over the uncertainty defined over a family of probability distributions.

        The choice of selecting $\mathcal{A}$, however, remains a central challenge, as finding admissible distributions that characterize $\theta$ is difficult to obtain.
        In DRO theory \cite{Ben-Tal:RobustOptimization,Rahimian_2022,kuhn2025distributionallyrobustoptimization}, $\mathcal{P}$ is defined as the following ambiguity set,
        \begin{equation}\label{eq:dist_family_KL}
            \mathcal{P} = \left\{p \ | \ \mathbb{D}_{KL} \left[ p \| q\right] \leq \epsilon\right\}
        \end{equation}
        where $\mathbb{D}_{KL}[p\,\|\,q] := \int_{\Theta}\log\big(\tfrac{dp}{dq}\big)\,dp$ denotes the Kullback--Leibler (KL) divergence between $p,q\in\mathcal{P}(\Theta)$ with $p$ absolutely continuous with respect to the unknown target distribution $q$ that captures the realization of the true parameter $\theta^\ast$.
        The surrogate $p$ is \emph{KL-close} to $q$ whenever $\mathbb{D}_{KL}[p\,\|\,q]\le\epsilon$.
        Since finding $\mathcal{P}$ such that element $p$ remains as KL-close to the target unknown distribution is challenging to obtain, DRO relaxes the min-max problem defined in \eqref{eq:minmax_original} to be the Risk-Averse Optimal Control Problem \cite{Ian1973OptimalSL},
        \begin{equation}
            \label{eq:softdro}
            \begin{split}
                \min_{x_{0:t_h},u_{0:t_h}, \lambda > 0} &\lambda \epsilon + \lambda \log \mathbb{E}_q \Bigg[ \exp\Big(\frac{1}{\lambda} \mathcal{J}(x_{0:t_h},u_{0:t_h}, \theta)\Big) \Bigg] \\
                \textrm{s.t. } &
                \textrm{constraints in \eqref{eq:task_opt}}
            \end{split}
        \end{equation}
        where $\lambda \in \Lambda$ are the dual variables to \eqref{eq:minmax_original}.
        We can now show that the entropic objective above reduces to an expected-cost problem for sufficiently large $\lambda$.
        \begin{lemma}[Entropic risk expansion]\label{lem:entropic}
            Suppose $\theta \mapsto \mathcal{J}(x_{0:t_h},u_{0:t_h},\theta)$ is bounded on $\Theta$, with $\underline{J} \le \mathcal{J} \le \bar{J}$. 
            Then for any $q \in \mathcal{P}(\Theta)$,
            \begin{equation}\label{eq:entropic}
                \begin{split}
                    \lambda \log \mathbb{E}_q\!\bigg[&\exp\!\Big(\tfrac{1}{\lambda}\mathcal{J}(x_{0:t_h},u_{0:t_h},\theta)\Big)\bigg] \\
                    &= \mathbb{E}_q\!\left[\mathcal{J}\right]
                    + \frac{1}{2\lambda}\,\mathbb{V}_q\!\left[\mathcal{J}\right]
                    + \mathcal{O}(\lambda^{-2}),
                \end{split}
            \end{equation}
            where the first-order remainder is uniformly bounded as $\frac{1}{2\lambda}\mathbb{V}_q[\mathcal{J}] \le \frac{(\bar{J}-\underline{J})^2}{8\lambda}$.
        \end{lemma}
        \begin{proof}
            Let $g(s) := \log \mathbb{E}_q[\exp(s\,\mathcal{J})]$ denote the cumulant generating function of $\mathcal{J}$ under $q$, which is finite and smooth for bounded $\mathcal{J}$, with $g(0)=0$, $g'(0)=\mathbb{E}_q[\mathcal{J}]$, and $g''(0)=\mathbb{V}_q[\mathcal{J}]$. 
            A second-order Taylor expansion of $g$ about $s=0$ evaluated at $s=1/\lambda$ yields
            \begin{equation}
                \lambda g(1/\lambda) = \mathbb{E}_q[\mathcal{J}] + \frac{1}{2\lambda}\mathbb{V}_q[\mathcal{J}] + \mathcal{O}(\lambda^{-2}).
            \end{equation}
            The uniform bound follows from Popoviciu's variance inequality \cite{NAM_1880_2_19__224_1}, 
            \begin{equation}
                \mathbb{V}_q[\mathcal{J}] \le \frac{1}{4}(\bar{J}-\underline{J})^2.
            \end{equation}
        \end{proof}

        In prior work, constructing $\mathcal{P}$ requires a maximum-likelihood estimate of target $q$ \cite{liu2025datadrivendistributionallyrobustoptimal} and a KL-neighborhood around it, demanding parameter-space coverage and repeated system sampling \cite{9661376} that is unavailable for latent $\theta$. 
        We instead assume minimal system feedback and shape uncertainty that is most critical to $\mathcal{J}$ via variational inference, avoiding explicitly choosing $\mathcal{P}$ entirely.
        
\section{Stein Variational Inference}

    Given a random variable $\theta \sim p_0(\theta)$ with prior law $p_0 \in \mathcal{P}(\Theta)$, Variational Inference (VI) offers a powerful method that approximates an intractable, unknown target distribution $q$ by optimizing a surrogate $p \in \mathcal{P}(\Theta)$ chosen from a family of distributions that satisfy \eqref{eq:dist_family_KL}, through the following,
    \begin{equation}
        \begin{split}
            p^*(\theta) &= \arg \min_{p\in\mathcal{P}} \{\mathbb{D}_{KL} (p \| q) \\
            &\equiv  \mathbb{E}_{\theta \sim p(\theta)} \big[\log p(\theta)\big] - \mathbb{E}_{\theta \sim p(\theta)} \big[\log q(\theta)\big] + \log c\}
        \end{split}
    \end{equation}
    where $p^*$ is the approximation of $q$, and $c$ is a constant that is often negligible in practice.
    As aforementioned, the choice of $\mathcal{P}$ remains a nontrivial challenge, and can negatively impact the performance of VI.
    As a solution, Stein Variational Inference, also referred to as Stein Variational Gradient Descent (SVGD) \cite{liu2019steinvariationalgradientdescent}, avoids the need of explicitly choosing $\mathcal{P}$ entirely.
    Instead, SVGD initializes a set of particles $\{\theta_0^i\}_{i=1}^N \sim p_0(\theta)$, and subsequently evolves them according to the mapping,
    \begin{equation}\label{eq:svgd_eq}
        \theta^i_{t+1} \leftarrow \theta^i_t + \alpha\,\phi^*_{p,q}(\theta^i_t)
        \qquad \forall t \in \{0, \dots, T-1\}
    \end{equation}
    where $\alpha$ is a step size parameter and $\phi_{p,q}^* (\theta_t^i)$ is a smooth function that characterizes the $t^{th}$ steepest descent direction that minimizes the KL-divergence measure, and $T$ represents the total number of SVGD iterations.
    The gradient $\phi^*_{p,q}$ is the solution to the following steepest descent problem,
    \begin{equation}\label{eq:steepest_desc_opt}
        \begin{split}
            \phi_{p,q}^*(\cdot) &= \arg \min_{\phi \in \mathcal{H}^d} \{ -\nabla_\xi \mathbb{D}_{KL} (p \| q) \ | \ \| \phi \|_{\mathcal{H}^d} \leq 1 \} \\
            &=\mathbb{E}_{\theta \sim p} \big[\mathcal{A}_q \mathcal{k}(\theta,\cdot) \big]
        \end{split}
    \end{equation}
    where $\mathcal{A}_q(\cdot): \Xi \rightarrow \mathcal{S}_\Xi$ is Stein's identity \cite{Stein1972} computed for a universal positive definite kernel function $k: \Xi \times \Xi \rightarrow \mathbb{R}$ operating in a dense $\mathcal{H}^d$ in the space of continuous functions $C(\Xi,\mathbb{R}^d)$, where $\mathcal{H}^d$ is the corresponding Reproducing Kernel Hilbert Space (RKHS).
    The closed form solution to \eqref{eq:steepest_desc_opt} is computed to be,
    \begin{equation}
        \begin{split}
            \phi^*_{p,q}(\cdot) = \mathbb{E}_{\theta\sim p}\!\left[k(\theta,\cdot)\nabla_{\theta}\log p'(\theta) + \nabla_{\theta} k(\theta,\cdot)\right]
        \end{split}
    \end{equation} 
    where $p$ denotes the law of the current particles and $p'$ is the tractable Boltzmann approximation to the target $q$, 
    \begin{equation}
        p'(\theta) = \frac{p(\mathcal{O} | \theta) p_0(\theta)}{\int_\Theta p(\mathcal{O}|\theta)p_0(\theta) d\theta}
    \end{equation}
    where $\int_\Theta p(\mathcal{O}|\theta)p_0(\theta)d\theta = \textrm{constant}$ and $\mathcal{O}$ is an observation output denoting an optimality criterion.
    Commonly, $\phi_{p,q}^*: \Theta \rightarrow \mathcal{S}_\Theta$ is approximated using Monte-Carlo samples over $\theta$,
    \begin{equation}
        \phi_{p,q}^*(\cdot) \approx \frac{1}{N} \sum_{i=1}^N k(\theta_t^i,\cdot) \nabla_\theta \log p'(\theta_t^i) + \nabla_\theta k(\theta_t^i,\cdot)
    \end{equation}
    which are initialized at random and then updated deterministically.
    With sufficient samples $N$ and number of SVGD iterations, the evolved particles become a sufficient approximation of the target $q(\theta)$ and show provably strong convergence \cite{korba2021nonasymptoticanalysissteinvariational} under regularity conditions.

\section{Parameter-induced Optimality Gap as Inference}

    This section outlines an approach that integrates the theoretical benefits of DRO combined with diverse, task-based parameter adaptation techniques via SVGD.
    We first define the Lagrangian of the optimization problem for a single planning cycle defined in \eqref{eq:task_opt} as,
    \begin{equation}
        \label{eq:mc_gap}
        \begin{split}
            \mathcal{L}(x_{0:t_h},u_{0:t_h},\theta) &= \mathcal{J}(x_{0:t_h},u_{0:t_h},\theta) + \beta^\top h(x_{0:t_h},u_{0:t_h},\theta) \\
            \textrm{where } &h(x_{0:t_h},u_{0:t_h},\theta) = \begin{Bmatrix}
                x_0 - \bar{x}_0 \\
                x_{t+1} - f(x_t,u_t,\theta)
            \end{Bmatrix}
        \end{split}
    \end{equation}
    where $\beta$ comprise the dual equality variables, $h(\cdot)$ are the equality constraints, and $x_t\in \mathcal{X}$ and $u_t \in \mathcal{U}$ are enforced.
    \begin{assumption}[Lagrangian $\Theta$-Smoothness]\label{assum:lagrangian} 
        $\mathcal{L}$ is bounded and continuous on the compact set $\Theta$, and
        there exists $\bar{\epsilon} < \infty$ such that
        \begin{equation}
            \sup_{x \in \mathcal{X},\, u \in \mathcal{U},\, \theta \in \Theta}
            \|\nabla_{\theta}\mathcal{L}(x,u,\theta)\| \le \bar{\epsilon}.
        \end{equation}
    \end{assumption}
    Because $\theta$ induces uncertainty about closed loop performance optimized by Lagrangian $\mathcal{L}: \mathcal{X} \times \mathcal{U} \times \Theta \rightarrow \mathbb{R}$, we define the optimality gap produced by a random variable $\theta$ as,
    \begin{equation}
        \delta \mathcal{L}(\cdot,\theta) = \mathcal{L}(x^*,u^*, \theta) - \mathcal{L}(x^*, u^*, \theta^*)
    \end{equation}
    where $\theta^*$ is the true parameter value.
    Referring back to \eqref{eq:softdro} and applying Lemma~\ref{lem:entropic} with $\mathcal{J}$ replaced by the Lagrangian $\mathcal{L}$ in \eqref{eq:mc_gap} (which coincides with $\mathcal{J}$ on dynamics-consistent trajectories, since $h(\cdot)=0$), the expected Lagrangian is approximated via Monte-Carlo expectation as,
    \begin{equation}\label{eq:exp_approx}
        \begin{split}
            \min_{x_{0:t_h},u_{0:t_h}} &\mathbb{E}_q \big[ \mathcal{L}(x_{0:t_h},u_{0:t_h}, \theta) \big]  \approx \frac{1}{N} \sum_{i=1}^N \mathcal{L}(x_{0:t_h},u_{0:t_h}, \theta^i) \\
            &= \mathcal{L} (\tau, \theta)\big|_{\theta=\theta^*} +  \gamma \ \frac{1}{N} \sum_{i=1}^N \delta \mathcal{L}(\tau, \theta^i)
        \end{split}
    \end{equation}
    where $\gamma$ is a design parameter such that $\gamma = 1$ equates to the true Monte-Carlo approximation above. 
    However, in practice, tuning $\gamma$ is desirable in order to control the trade-off between optimizing over the nominal parameters versus the variations of parameters.

    Note that $\theta^*$ is unknown a-priori, and requires an additional approximation in order to implement \eqref{eq:exp_approx}. We choose the empirical posterior mean (particle mean) as a standard and consistent estimator in particle-based adaptation used in Stein-based control \cite{lee2024steinvariationalergodicsearch,lambert2021steinvariationalmodelpredictive},
    \begin{equation}
        \label{eq:particle_avg}
        \theta^*\approx \bar{\theta} = \frac{1}{N} \sum_{i=1}^N \theta^i
    \end{equation}
    which is shown to provide a stable, low-variance reference for defining the
    optimality gap.
    The optimization in \eqref{eq:exp_approx} is commonly used in current robust control methods \cite{Abraham_2020, barcelos2021dualonlinesteinvariational}, but a core challenge in implementation remains in \emph{how the parameter samples $\{ \theta^i \}_{i=1}^N$ are chosen}.
    

    \begin{figure*}[ht!]
        \centering
        \subfloat[\textbf{Rocket example.} Given its initial condition, the rocket system navigates to a landing pad safely. 
        The top-heavy rocket system requires explicit reasoning over this uncertainty by modulating gimbal and thrust to maintain a stable configuration under uncertainty about its mass, inertia, and center of mass location.]{
            \includegraphics[width=0.99\textwidth]{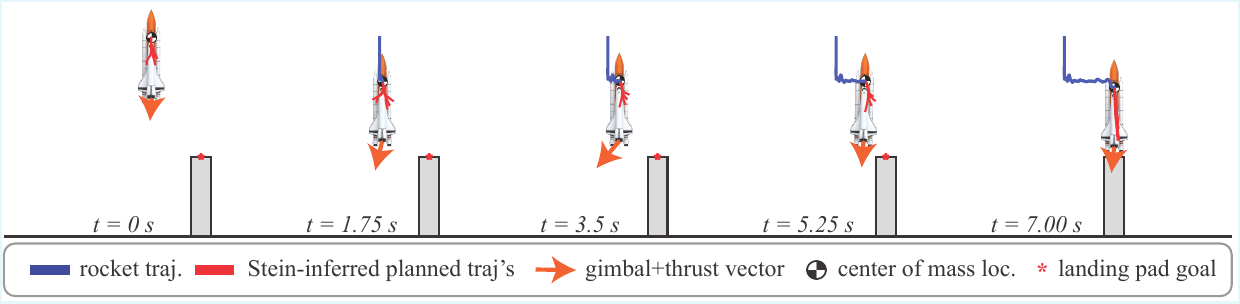}
            \label{fig:rocket_example_traj}
        }
        \hfill
        \vspace{-1em}
        \subfloat[\textbf{Cartpole example.} A cartpole system achieves a swingup task from a downward-stable state.
        The system exploits oscillatory motions until SVGD converges to an approximate posterior, then robustly stabilizes the pole to the goal state under uncertainty over pole mass and length.]{
            \includegraphics[width=0.99\textwidth]{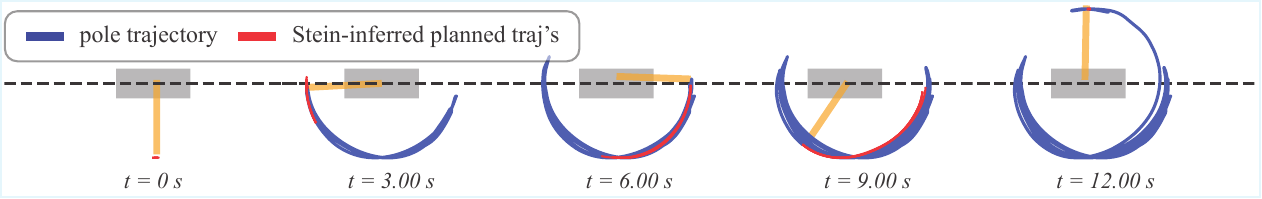}
            \label{fig:cartpole_example_traj}
        }
        \caption{\textbf{Example Experimental Outcomes from Our Method. } We demonstrate the efficacy of our approach on a two-dimensional rocket landing and cartpole swingup control tasks.}
        \label{fig:combined_examples}
        \vspace{-1.5em}
    \end{figure*}

    \subsection{Stein Variational Uncertainty Adaptation}

        We outline a principled, non-parametric approach that guides parameter particles $\theta^i$ towards task-critical uncertainties by contextualizing SVGD as an inference problem over uncertain parameters. 
        We define the log posterior distribution given a parameter prior $p_0(\theta)$ as a Boltzmann posterior,
        \begin{equation}\label{eq:log_posterior}
            \begin{split}
                \log p'(\theta) &= \log \Bigg( \frac{\exp \big( \delta \mathcal{L}(x_{0:t_h},u_{0:t_h}, \theta)\big) p_0(\theta)}{z} \Bigg) \\
                &= \log p_0(\theta) + \delta \mathcal{L}(x_{0:t_h},u_{0:t_h}, \theta) - \log z
            \end{split}
        \end{equation}
        where $p(\mathcal{O}|\theta)$ is the observation likelihood via an exponentiated optimality gap $p(\mathcal{O}\mid\theta) \propto \exp(\delta\mathcal{L}(x_{0:t_h},u_{0:t_h},\theta))$, commonly used in relevant Stein literature \cite{lee2024steinvariationalergodicsearch,lambert2021steinvariationalmodelpredictive}, in which we explicitly characterize model optimality variations induced by uncertain parameters that directly impacts the Stein flow.
        Also note that $z$ is a constant that is often ignored in the optimization.
        We utilize this posterior approximation to the target $q$ by evolving initialized Stein particles $\{ \theta^i\}_{i=1}^N$ by computing steepest descent $\phi^*$ using Stein's identity given prior $p_0(\theta)$,
        \begin{equation}\label{eq:steepest_desc_params}
            \begin{split}
                \phi^*(\cdot) &= \mathbb{E}_{\theta \sim p(\theta)} \big[ k(\theta,\cdot) \nabla_\theta \log p'(\theta) + \nabla_\theta k(\theta,\cdot) \big] \\
                &\approx \frac{1}{N} \sum_{i=1}^N k(\theta^i,\cdot)\nabla_{\theta} \log p'(\theta^i) + \nabla_{\theta} k(\theta^i,\cdot)
            \end{split}
        \end{equation}
        where $k:\Theta \times \Theta \rightarrow \mathbb{R}$ is a universal positive definite kernel.
        Our approach evaluates parameter sensitivity through $t^{th}$ evolving estimators $\{ \theta_t^i\}_{i=1}^N$, capturing the task-relevant deviations $\delta \mathcal{L}(x_{0:t_h},u_{0:t_h}, \theta)$ and enabling control synthesis that explicitly counteracts the induced optimality gap. 
        Note that Algorithm~1 performs a single SVGD update per planning cycle, so the Stein iteration index coincides with the time step $t$ and the total number of SVGD iterations equals the task duration $T$.

        \begin{remark}[Consistency with the DRO relaxation]
            \label{rem:dro_consistency}
            By the Donsker-Varadhan variational formula \cite{doi:10.1073/pnas.72.3.780}, for any
            bounded $\mathcal{L}$ and any reference measure $p_0 \in \mathcal{P}(\Theta)$,
            \begin{equation}
                \label{eq:dv}
                \begin{split}
                    \lambda &\log \mathbb{E}_{p_0}\!\left[\exp\!\Big(\tfrac{1}{\lambda}
                    \mathcal{L}(\tau,\theta)\Big)\right] \\
                    &= \sup_{p \,\ll\, p_0}\Big\{\mathbb{E}_{p}\big[\mathcal{L}(\tau,\theta)\big]
                - \lambda\,\mathbb{D}_{KL}\big[p \,\|\, p_0\big]\Big\},
                \end{split}
            \end{equation}
            where the supremum is attained at $\tfrac{dp^{\star}}{dp_0}(\theta) \propto \exp\big(\mathcal{L}(\tau,\theta)/\lambda\big)$. 
            The left-hand side of \eqref{eq:dv} is, up to the constant $\lambda\epsilon$, the entropic objective of \eqref{eq:softdro}, while the right-hand side is the Lagrangian form of the inner adversarial maximization in \eqref{eq:minmax_original}. 
            Since $\delta\mathcal{L}(\tau,\theta)$ and $\mathcal{L}(\tau,\theta)$ differ by a $\theta$-independent constant, \eqref{eq:log_posterior} coincides with the least-favorable distribution $p^{\star}$ at $\lambda = 1$.
            Coefficient $\lambda$ merely scales the exponent in \eqref{eq:log_posterior} by $1/\lambda$. 
            The Stein flow \eqref{eq:steepest_desc_params} therefore targets precisely the maximizer of the KL-regularized distributionally robust objective by transporting particles toward  highly sensitive regions of the task-based posterior as the ensemble-based controller seeks the best response against the particles via \eqref{eq:exp_approx}.
        \end{remark}

        \begin{algorithm}
        \caption{Stein Variational Uncertainty-Adaptive MPC}\label{alg:sv_dro}
            \begin{algorithmic}[1]
            \Require planning horizon $t_h$, initial trajectory $ (x_{0:t_h}, u_{0:t_h})$ parameter prior $p_0(\theta)$, $N$ total number of particles, SVGD step size $\alpha$, design parameter $\gamma$, initial particles $\{\theta_0^i\}_{i=1}^N \sim p_0(\theta)$, dynamics $f(x_t,u_t,\theta)$, total time duration $T$.
            \State $t = 0$
            \While{$t< T$}
                \State $\tau^* \gets \textrm{\eqref{eq:exp_approx}}$
                \State \text{Apply first control $u^*[0]$ via MPC}
                \State $x_{t+1} \gets f(x_t,u^*[0],\theta)\big|_{\theta=\theta^*}$ \Comment{Env. Step}
                \State $\{\theta_{t+1}^i\}_{i=1}^N \gets$ SVGD($\{\theta_t^i\}_{i=1}^N $) \Comment{\eqref{eq:svgd_eq}}
                \State $t \gets t + 1$
            \EndWhile
            \end{algorithmic}
        \end{algorithm}
        
        \begin{lemma}[Continuity of the Stein potential]\label{lemma:V}
            Given a feasible state-control pair $(x_{0:t_h},u_{0:t_h})$, let
            \begin{equation}
                V(\theta):=-\delta \mathcal{L}(x_{0:t_h},u_{0:t_h},\theta)-\log p_0(\theta),
                \qquad \theta \in \Theta,
            \end{equation}
            where
            \begin{equation}
                \delta \mathcal{L}(x_{0:t_h},u_{0:t_h},\theta)= \mathcal{L}(x_{0:t_h},u_{0:t_h},\theta)-\mathcal{L}(x_{0:t_h},u_{0:t_h},\theta^\ast).
            \end{equation}
            Suppose Assumption \ref{assum:lagrangian} holds and $p_0(\theta)$ is strictly positive and continuous on $\Theta$. Then $V$ is continuous on $\Theta$.
            \end{lemma}
        \begin{proof}
            Since $\theta^\ast$ is fixed, the scalar
            \begin{equation}
                \mathcal{L}(x_{0:t_h},u_{0:t_h},\theta^\ast)
            \end{equation}
            is constant with respect to $\theta$. Therefore,
            \begin{equation}
                \theta \mapsto \delta \mathcal{L}(x_{0:t_h},u_{0:t_h},\theta)
            \end{equation}
            is continuous on $\Theta$.
            
            Because $p_0$ is strictly positive and continuous on $\Theta$, the function $\log p_0(\theta)$ is continuous on $\Theta$.
            Therefore, $V(\theta)=-\delta \mathcal{L}(x_{0:t_h},u_{0:t_h},\theta)-\log p_0(\theta)$ is the sum of two continuous functions on $\Theta$, and is thus continuous.
        \end{proof}

        \begin{theorem}[Posterior Convergence in $\Theta$ \cite{korba2021nonasymptoticanalysissteinvariational}]\label{thm:svgd_convergence}
            Given an initial prior $p_0(\theta)$, let $V(\theta)= -\delta \mathcal{L}(x_{0:t_h},u_{0:t_h},\theta) - \log p_0(\theta)$ be the continuous (hence lower semi-continuous), task-dependent value function on the compact space $\Theta$, $k(\cdot,\theta)$ is a universal positive definite kernel in $\Theta$, and let $KSD_{p'}(\mu_i) = \| \phi^*(\mu_i)\|_{\mathcal{H}^\Theta}^2$ be the kernel Stein discrepancy over the posterior $p'$. Then, $\exists \kappa > 0$ such that,
            \begin{equation}\label{eq:convergence}
                \frac{1}{N} \sum_{i=1}^N\textrm{KSD}_{p'}(\mu_i) \leq \frac{\mathbb{D}_{KL}(p_0\|p')}{Nc_\kappa}
            \end{equation}
            where $c_\kappa$ is a $\kappa$-parameterized term \cite[Corollary 6]{korba2021nonasymptoticanalysissteinvariational}.
        \end{theorem}

        Thus, for sufficient $T$ number of SVGD iterations, the approximated posterior $p'$ sufficiently represents the intractable target posterior $q$ which we show in Section~\ref{sec:results}.

        \begin{theorem}[Convergence of optimality gap]
            Let $q \in \mathcal{P}(\Theta)$ denote the intractable target posterior over parameters, and let $\{ p_T'\}_{T > 0} \subset \mathcal{P}(\Theta)$ denote the sequence of empirical Stein particle measures (posteriors) generated by the SVGD updates for $T$ iterations in \eqref{eq:svgd_eq}.
            Given Assumption \ref{assum:set_assumptions} holds true, the conditions in Theorem \ref{thm:svgd_convergence} are satisfied, and the mapping $\theta \mapsto \mathcal{L}(x_{0:t_h},u_{0:t_h},\theta)$ is bounded and continuous for fixed state control pairs $(x_{0:t_h},u_{0:t_h})$. Then for any $\varepsilon > 0$, $\exists T \in \mathbb{Z}_+$ such that,
            \begin{equation}
                \Big| \mathbb{E}_{\theta \sim q(\theta)} \big[ \mathcal{L}(\cdot,\theta)\big] - \mathbb{E}_{\theta \sim p_T'(\theta)} \big[ \mathcal{L}(\cdot,\theta)\big] \Big| \leq \varepsilon
            \end{equation}
        \end{theorem}
        \begin{proof}
            Since the Stein kernel $k$ is universal on compact $\Theta$, Theorem \ref{thm:svgd_convergence} shows convergence on the associated empirical measures,
            \begin{equation}
                p_T' \Rightarrow q \ \textrm{ as } \ T \rightarrow \infty.
            \end{equation}
            For a fixed trajectory $(x_{0:t_h},u_{0:t_h})$, we define the following function,
            \begin{equation}
                f(\theta) := \mathcal{L}(x_{0:t_h},u_{0:t_h},\theta).
            \end{equation}
            By Assumption \ref{assum:lagrangian}, $f$ is bounded and continuous on $\Theta$.
            Therefore, invoking Portmanteau Theorem \cite{billing},
            \begin{equation}
                \int_\Theta f(\theta) p'_T(\theta) d\theta \rightarrow \int_\Theta f(\theta) q(\theta) d\theta \ \textrm{ as } \ T \rightarrow \infty
            \end{equation}
            and equivalently,
            \begin{equation}
                \mathbb{E}_{\theta \sim p_T'(\theta)} \big[ \mathcal{L}(\cdot,\theta)\big] \rightarrow \mathbb{E}_{\theta \sim q(\theta)} \big[ \mathcal{L}(\cdot,\theta)\big].
            \end{equation}
            Hence,
            \begin{equation}
                \Big| \mathbb{E}_{\theta \sim q(\theta)} \big[ \mathcal{L}(\cdot,\theta)\big] - \mathbb{E}_{\theta \sim p_T'(\theta)} \big[ \mathcal{L}(\cdot,\theta)\big] \Big| \rightarrow 0 \ \textrm{ as } \ T \rightarrow \infty.
            \end{equation}
            By the definition of convergence of a sequence, for any $\varepsilon > 0$, $\exists T \in \mathbb{Z}_+$,
            \begin{equation}
                \Big| \mathbb{E}_{\theta \sim q(\theta)} \big[ \mathcal{L}(\cdot,\theta)\big] - \mathbb{E}_{\theta \sim p_T'(\theta)} \big[ \mathcal{L}(\cdot,\theta)\big] \Big| \leq \varepsilon
            \end{equation}
            thus completing the proof.
        \end{proof}

        \begin{prop}[Performance per-cycle upper bound]
            \label{prop:ce_bound}
            Let $\gamma \in [0,1]$ in \eqref{eq:exp_approx}, and let Assumptions \ref{assum:set_assumptions} and \ref{assum:lagrangian} hold with $\Theta$ convex, so that $\mathcal{L}(\tau,\cdot)$ is $\bar{\epsilon}$-Lipschitz on $\Theta$ for every feasible $\tau$. 
            Let $p'$ denote the empirical Stein particle measure with mean $\bar{\theta}$ as in \eqref{eq:particle_avg}, let $\tau^{\star}$ denote the minimizer of \eqref{eq:exp_approx}, and let $\hat{\tau}(\hat{\theta})$ denote the certainty-equivalent MPC solution for an arbitrary fixed $\hat{\theta} \in \Theta$ over the same feasible set. Then,
            \begin{equation}
                \begin{split}
                    \mathcal{L}(\tau^{\star}, \theta^{*})
                    &\le \mathcal{L}\big(\hat{\tau}(\hat{\theta}), \theta^{*}\big)
                    + 2\bar{\epsilon}\rho, \\
                    \textrm{where }\rho &:= \mathbb{E}_{\theta\sim p'}\big[\|\theta - \theta^{*}\|\big]
                    \le \mathrm{diam}(\Theta).
                \end{split}
            \end{equation}
        \end{prop}
        \begin{proof}
            The objective of \eqref{eq:exp_approx} can be rewritten as
            \begin{equation}
                F(\tau) = (1-\gamma)\,\mathcal{L}(\tau,\bar{\theta}) + \frac{\gamma}{N}\sum_{i=1}^{N}\mathcal{L}(\tau,\theta^{i}).
            \end{equation}
            By the Lipschitz property, $|\mathcal{L}(\tau,\theta^{i}) - \mathcal{L}(\tau,\theta^{*})| \le \bar{\epsilon}\|\theta^{i} - \theta^{*}\|$, and by Jensen's inequality $\|\bar{\theta} - \theta^{*}\| \le \rho$. 
            Hence,
            \begin{equation}
                |F(\tau) - \mathcal{L}(\tau,\theta^{*})| \le \bar{\epsilon}\rho
            \end{equation} 
            for feasible $\tau$ and $\gamma \in [0,1]$. Since $\hat{\tau}(\hat{\theta})$ is feasible and $\tau^{\star}$ minimizes $F$,
            \begin{equation}
                \mathcal{L}(\tau^{\star},\theta^{*})
                \le F(\tau^{\star}) + \bar{\epsilon}\rho
                \le F\big(\hat{\tau}(\hat{\theta})\big) + \bar{\epsilon}\rho
                \le \mathcal{L}\big(\hat{\tau}(\hat{\theta}),\theta^{*}\big)
                + 2\bar{\epsilon}\rho.
            \end{equation}
        \end{proof} 
        \begin{remark}
            Per planning cycle, the controller performs no worse than certainty-equivalent MPC with any, possibly incorrect, $\hat{\theta}$, up to slack $2\bar{\epsilon}\rho$ governed by posterior concentration around $\theta^*$; taking $\hat{\theta} = \theta^*$ shows the cost approaches that of true-parameter MPC as $\rho \to 0$. 
            The bound holds at every cycle, hence recursively along the receding-horizon execution.
        \end{remark}
        
    \begin{table}
        \centering
        \footnotesize
        \setlength{\tabcolsep}{4pt}
        
        \begin{tabular}{lcc|cc}
            \toprule
            & \multicolumn{2}{c}{\textit{Cartpole Swing-Up}} 
            & \multicolumn{2}{c}{\textit{Rocket Landing}} \\
            \cmidrule(lr){2-3} \cmidrule(lr){4-5}
            \textbf{Method} 
            & \textbf{Success (\%)} & \textbf{Time (s)} 
            & \textbf{Success (\%)} & \textbf{Time (s)} \\
            \midrule
            \textbf{Ours} 
            & \textbf{100.0} & \textbf{10.24 $\pm$ 2.57} 
            & \textbf{90.63}  & 5.70 $\pm$ 2.53 \\
            EMPPI 
            & 93.75 & 24.47 $\pm$ 13.82 
            & 68.75 & \textbf{5.21 $\pm$ 2.54} \\
            MPC   
            & 93.75 & 23.54 $\pm$ 13.84 
            & 56.25 & 5.52 $\pm$ 2.81 \\
            DRO   
            & 78.13 & 23.34 $\pm$ 13.45 
            & 68.75 & 5.25 $\pm$ 3.38 \\
            \bottomrule
        \end{tabular}
        
        \caption{Performance comparison on cartpole swing-up and rocket landing tasks.
        Our method consistently achieves the highest success rates on both problems, 
        while also exhibiting substantially faster convergence on the cartpole task ($\approx2\times$ relative speedup), at the occasional slight cost of increased completion time.}
        \label{tab:control_results}
        \vspace{-2em}
    \end{table}
\section{Results}\label{sec:results}

    We evaluate the proposed controller on three nonlinear control problems with latent parameters under broad uncertainty with distinct tasks:
    \begin{enumerate}
        \item \textbf{Cartpole Swingup.} Swing the pole upright from a downward-oriented initial state.
        \item \textbf{Rocket Landing.} Navigate a 2D rocket to a landing pad.
        \item \textbf{Autonomous Racing.} Complete a lap of a 2D track as quickly as possible.
    \end{enumerate}
    All control problems are given the following general stage and terminal cost structure based on \eqref{eq:task_opt},
    \begin{equation}\label{eq:specified_stage_terminal_costs}
        \begin{split}
            \ell(x_t,u_t,\theta) &= (x_t - x_{des})^\top \mathbf{Q} (x_t - x_{des})
            + u_t^\top \mathbf{R} u_t\\
            m(x_{t_h}, u_{t_h}, \theta) &= (x_{t_h} - x_{des})^{\top}\mathbf{Q_f}(x_{t_h} - x_{des})
            + r(\cdot, \theta)
        \end{split}
    \end{equation}
    where $\mathbf{Q},\mathbf{R},\mathbf{Q_f} \succeq 0$, and $r(\cdot)$ is a problem-specific additional terminal cost.
    Note $r(\cdot)=0$ unless otherwise specified.

    \begin{table*}[t]
        \centering
        \setlength{\tabcolsep}{3pt}
        \begin{tabular}{l c c c}
        \toprule
         & Cartpole & Rocket & Racing \\
        \midrule
        $(n_x,\, n_u,\, n_\theta)$ & $(4,1,2)$ & $(6,2,3)$ & $(5,2,2)$ \\
        $\theta$ & $[m_{\mathrm{pole}}, \ell_{\mathrm{pole}}]$
                 & $[m_{\mathrm{r}}, I_{\mathrm{r}}, \ell_{\mathrm{COM}}]$
                 & $[m_{\mathrm{car}}, I_{\mathrm{car}}]$ \\
        $\theta^{*}$ & $[0.5,\, 0.75]$ & $[0.1,\, 0.01,\, 0.7]$ & $[0.1,\, 0.01]$ \\
        $(\theta_{\min},\theta_{\max})$ & $([0.3,\, 0.3], [1.0,\, 1.0])$ & $([0.05,\, 0.005,\, 0.05], [5.0,\, 2.0,\, 1.0])$ & $([0.05,\, 10^{-5}],[0.3,\, 0.5])$ \\
        $\mathbf{x}_0$ & $[0,0,0,0]$ & $[0, 0.5, 0, 0, 0, 0]$ & $[-2.5, -2, 0, 0, 0]$ \\
        $\mathbf{x}_{\mathrm{des}}$ & $[0, \pi, 0, 0]$ & $[0.5, 0, 0, 0, 0, 0]$ & ref.\ spline \\
        $t_h$ (s) & $0.4$ & $0.15$ & $0.15$ \\
        $\alpha$ & $10^{-3}$ & $10^{-3}$ & $10^{-2}$ \\
        $\mathbf{Q}$ & $\mathrm{diag}([0.01, 2.0, 0.01, 0.15])$ & $\mathrm{diag}([5.0, 5.0, 5.0, 0.001, 0.001, 0.01])$ & $\mathrm{diag}([5.0, 5.0, 0.001, 0.1, 0.01])$ \\
        $\mathbf{R}$ & $0.0001$ & $\mathrm{diag}([0.0001, 0.0001])$ & $\mathrm{diag}([0.1, 0.1])$ \\
        $\mathbf{Q_f}$ & $\mathrm{diag}([0.001, 0.001, 0.001, 2.0])$ & $\mathrm{diag}([5.0, 5.0, 5.0, 0.001, 0.001, 0.01])$ & $\mathrm{diag}([5.0, 5.0, 0.001, 0.1, 0.01])$ \\
        \bottomrule
        \end{tabular}
        \caption{Implementation details. All ensemble-based experiments use
        $N = 5$ particles, $\gamma = 0.5$,
        and uniform priors
        $\{\theta^i\}_{i=1}^{N} \sim \mathcal{U}(\theta_{\min}, \theta_{\max})$.}
        \label{tab:impl}
        \vspace{-1.5em}
    \end{table*}


        We compare our Stein Variational uncertainty-adaptive Model Predictive Control to the following three comparative baselines:
        \begin{enumerate}
            \item \textbf{Ensemble MPC (EMPPI)
            \cite{Abraham_2020,lowrey2019planonlinelearnoffline}.} Samples control sequences and propagates them across a parameter ensemble via task-agnostic (uniform random) sampling.
            \item \textbf{Standard DRO \cite{kuhn2025distributionallyrobustoptimization,liu2025datadrivendistributionallyrobustoptimal}.} Optimizes the ambiguity-set objective \eqref{eq:minmax_original} against a worst-case parameter distribution.
            \item \textbf{Standard MPC.} Solves the finite-horizon problem with the nominal parameter, ignoring parametric uncertainty.
        \end{enumerate}
        Each example is evaluated over $32$ trials with different seeds initializing the particles $\{\theta^i\}_{i=1}^{N} \sim \mathcal{U}(\theta_{\min}, \theta_{\max})$. Unless otherwise stated, we
        use the RBF kernel $k(\theta,\hat{\theta}) = \exp(-\|\theta-\hat{\theta}\|^2/h)$ with bandwidth $h = 1$.

        \emph{Computational cost.} 
        Per planning cycle, our method requires $N$ parallel trajectory rollouts to evaluate $\delta\mathcal{L}(\tau,\theta^i)$ per SVGD update, each requiring $\mathcal{O}(N^2 n_\theta)$ kernel evaluations, where $n_\theta=|\theta|$, for a total per-cycle complexity of $\mathcal{O}(N t_h + N^2 n_\theta)$, compared with $\mathcal{O}(t_h)$ for nominal MPC and $\mathcal{O}(N t_h)$ for EMPPI. The rollouts dominate and parallelize identically to EMPPI's ensemble, so the SVGD loop adds only marginal overhead for small $N$.
        Measured per-cycle solve times are reported in
        Table~\ref{tab:timing}, and all methods remain within the receding-horizon
        replanning budget.
        
        \begin{table}[t]
            \centering
            \begin{tabular}{l c c c}
            \toprule
            Method & Cartpole & Rocket & Racing \\
            \midrule
            Ours  & $27.2 \pm 2.3$ & $24.8 \pm 2.6$ & $26.9 \pm 2.3$ \\
            EMPPI & $23.1 \pm 2.1$ & $21.2 \pm 0.5$ & $21.2 \pm 2.1$ \\
            MPC   & $19.5 \pm 1.8$ & $18.3 \pm 1.2$ & $18.1 \pm 1.2$ \\
            DRO   & $22.5 \pm 1.2$ & $20.6 \pm 2.1$ & $22.7 \pm 2.3$ \\
            \bottomrule
            \end{tabular}
            \caption{Per-cycle solve times (ms, mean $\pm$ std across all planning
            cycles of the 32 trials per task) measured on the GPU (NVIDIA GeForce
            RTX 3080).}
            \label{tab:timing}
            \vspace{-1.5em}
        \end{table}

    \subsection{Cartpole Swingup}

        \subsubsection{Problem Formulation}
        Let $\mathbf{x^{cp}}_t = [x_t, \varphi_t, v_t, \omega_t]^{\top} \in \mathbb{R}^4$ denote the cart position, the pole angle measured from the downward vertical, and their respective velocities, with force input $\mathbf{u^{cp}}_t \in \mathbb{R}$ along the translational axis. 
        Starting from rest with the pole hanging down, the objective is to swing the pole upright. 
        The uncertain parameters are the pole mass $m_\textrm{pole} \in \mathbb{R}$ and length $\ell_\textrm{pole} \in \mathbb{R}$, and all numerical parameters are listed in Table~\ref{tab:impl}.

        \subsubsection{Results}
        Fig.~\ref{fig:cartpole_example_traj} shows an example trajectory of a cartpole exploiting oscillatory motions until SVGD converges to an approximate posterior of the target to robustly stabilize the pole to the goal state. 
        Results across all $32$ trials are reported in Table~\ref{tab:control_results}.

        Our method achieves $\approx 2\times$ faster completion with consistently lower variance across all parameter initializations. 
        EMPPI's \emph{task-agnostic} sampling drives inefficient exploration, while DRO enforces uniformly \emph{worst-case} robustness irrespective of task relevance. 
        Stein-based posterior updates instead reshape the parameter distribution via the optimality gap, suppressing task-irrelevant parameter variations and amplifying performance-critical ones, eliminating unnecessary exploration.
        

    \begin{figure}
        \centering
        \vspace{0.2em}
        \includegraphics[width=0.8\linewidth]{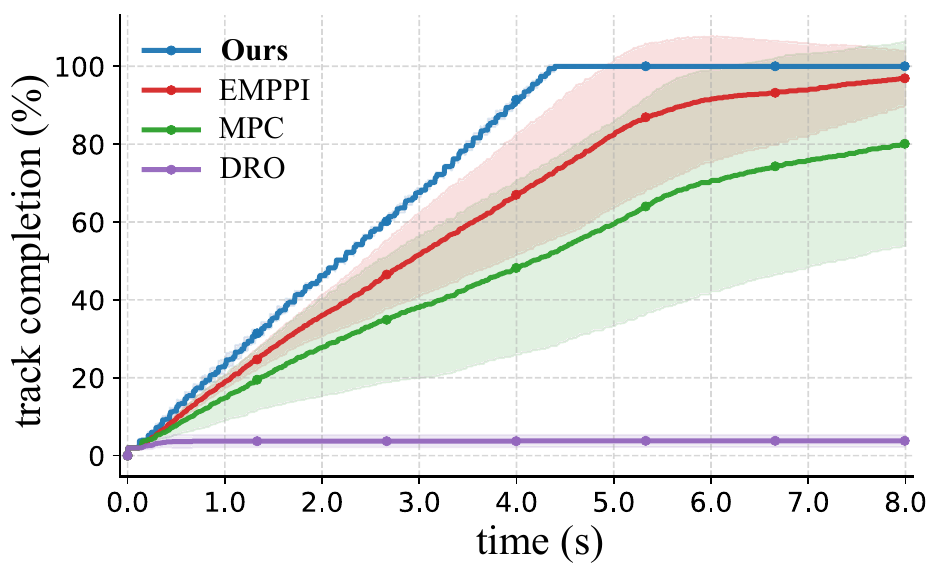}
        \caption{\textbf{Autonomous Racing Lap Time Completions. } We report track completion over time for the autonomous racing task under uncertainty in vehicle mass and inertia. 
        We achieve faster and more consistent progress by reasoning over task-sensitive regions of the parameter posterior that most affect control performance.} 
        \label{fig:race_completion_times}
        \vspace{-1.9em}
    \end{figure}
    \subsection{2D Rocket Landing}
        
        \subsubsection{Problem Formulation}
        Let $\mathbf{x^{r}}_t \in \mathcal{X}_t^{\textrm{rocket}}=\mathbb{R}^6$ be the rocket system state, defined as $\mathbf{x^{r}}_t = [x_t, y_t, \varphi_t, \dot{x}_t, \dot{y}_t, \omega_t]^\top$,
        where $x_t$ and $y_t$ are the translational world coordinates of the system, $\varphi_t$ is the measured rocket angle with respect to the world vertical axis (with world frame positioned at `*' on the landing pad shown in Fig. \ref{fig:rocket_example_traj}), and $\dot{x}_t$, $\dot{y}_t$, and $\omega_t$ are translational and angular velocities, respectively.
        Let $\mathbf{u^{r}}_t \in \mathcal{U}_{\textrm{rocket}}=\mathbb{R}^2$ be the thrust and gimbal control of the rocket, where the thrust is exerted along the gimbal axis.
        The goal state is placed a lateral distance from the rocket, so the system cannot trivially descend vertically and must instead execute careful lateral motion to land safely. The uncertain parameters are the rocket mass $m_\textrm{rocket}$, inertia $I_\textrm{rocket}$, and center-of-mass location $\ell_{\mathrm{COM}}$ measured along the body axis, with $\theta_{\max,3} = h_{\mathrm{rocket}} = 1.0$ the rocket height; all numerical parameters are listed in Table~\ref{tab:impl}.

        \subsubsection{Results}
        Fig.~\ref{fig:rocket_example_traj} shows an example trajectory, where the task succeeds only if the rocket lands within a safe tolerance of the pad with its body axis upright. 
        Since the goal is a lateral distance from the rocket, the system must move laterally without tipping the high-positioned center of mass, and the resulting behavior is a gradual gimbal-thrust forcing that translates the rocket without capsizing it. 
        Results across $32$ trials are reported in Table~\ref{tab:control_results}.
        \begin{table}
            \centering
            \footnotesize
            \setlength{\tabcolsep}{6pt}
            \renewcommand{\arraystretch}{1.1}
            \begin{tabular}{lcc}
                \toprule
                \textbf{Method} & \textbf{Success (\%)} & \textbf{Time (s)} \\
                \midrule
                IMQ & \textbf{95.0} & $7.63 \pm 3.41$ \\
                RBF          & 90.0          & \textbf{4.55 $\pm$ 2.19} \\
                $k(\cdot,\hat\theta) = 1$     & 80.0          & $6.48 \pm 3.55$ \\
                \bottomrule
            \end{tabular}
            \caption{\textbf{Kernel Ablation Study. }
            Comparison of our method for the rocket landing system using RBF, IMQ, and $k(\cdot,\hat{\theta})=1$ kernels.
            The IMQ kernel achieves the highest success rate with slightly slower completion due to the kernel's prevention of mode collapse and preservation of particle diversity for more reliable inference.}
            \label{tab:kernel_comparison}
            \vspace{-2.5em}
        \end{table}

        While EMPPI attains the lowest average completion time, it also has the lowest success rate among all methods. 
        Our method achieves the highest reliability at near-negligible added time. 
        By `probing' the rocket through subtle tilts that stop short of tipping, our approach generates the exploration SVGD needs to approximate the target posterior via task-sensitive inference.
        \subsection{Autonomous Racing}

        \subsubsection{Problem Formulation}
        Let $\mathbf{x^{c}}_t \in \mathcal{X}_t^{\textrm{car}}=\mathbb{R}^5$ be the vehicle system state, defined as $\mathbf{x^{c}}_t = [x_t, y_t, \varphi_t, v_t, \omega_t]^\top$,
        where $x_t$ and $y_t$ are the translational world coordinates of the system (with world frame oriented at the center of the racetrack), $\varphi_t$ is the measured vehicle angle with respect to the body forward axis, $v_t$ is the forward directional velocity with respect to the body axis, and $\omega_t$is the angular velocities, respectively.
        Let $\mathbf{u^{c}}_t \in \mathcal{U}_{\textrm{car}}=\mathbb{R}^2$ be the throttle and steering control of the vehicle.
        We define the dimensions of the racetrack as a track length of $5.0$ and a track radius of $2.0$. 
        We initialize the vehicle state in world coordinates at the starting line of the racetrack, $\mathbf{x^{c}}_0 = [-2.5, -2.0, 0.0, 0.0, 0.0]^\top$,
        with the objective of completing a single lap along the track as quickly as possible, where $\mathbf{x^{c}}_\textrm{des}$ is a reference spline along the central spline of the racetrack.
        To promote progress along the track, we add the terminal reward
        \begin{equation}
            r(x_{0:t_h}, u_{0:t_h}, \theta)
            = \sum_{j =1}^{n_x} \frac{1}
            {\big[\mathbf{x^c}_{t_h} - \mathbf{x^c}_{0}\big]_j + \varepsilon_r},
        \end{equation}
        where $[\cdot]_j$ denotes the $j$-th entry, and $\varepsilon_r = 0.001$ avoids
        division by zero.
        The uncertain parameters are the vehicle mass $m_\textrm{car}$ and inertia $I_\textrm{car}$. 
        Additional details are listed in Table~\ref{tab:impl}.

        \subsubsection{Results}
        Fig.~\ref{fig:abstract} shows example trajectories of our controller adapting to the target parameter posterior to produce dynamically consistent trajectories that track the reference spline, maintaining high-speed traversal with minimal corrective action and achieving the fastest lap time of $4.408$~s. 
        As shown in Fig.~\ref{fig:race_completion_times}, our method completes the track faster and with reduced variance relative to EMPPI and MPC, while standard DRO fails to reliably complete the track.
        EMPPI's task-agnostic sampling induces redundant exploration, nominal MPC neglects parameter uncertainty, and DRO's worst-case design yields overly conservative, unstable behavior. 
        Adapting to the task-dependent posterior via the optimality gap instead yields trajectories that are simultaneously aggressive, stable, and robust, achieving faster laps.
    \subsection{Kernel Ablation Study}
        To demonstrate kernel sensitivity to task performance, Table~\ref{tab:kernel_comparison} evaluates our approach on the rocket landing problem over three kernels: the RBF kernel used above, the Inverse Multi-Quadratic (IMQ) kernel \cite{gorham2020measuringsamplequalitykernels}, $k(\theta,\hat{\theta}) = (\psi^2 + \|\theta - \hat{\theta}\|^2)^{-\zeta}$ with bandwidth $\psi$ and decay factor $\zeta$, and the constant kernel $k(\cdot,\hat\theta)=1$, which reduces SVGD to parallel gradient descent over the parameters. 
        The IMQ kernel yields the most reliable control at the cost of slightly longer completion times.
        RBF's exponential decay limits long-range particle interactions and can induce clustering or degeneracy, whereas IMQ's long-range repulsion preserves particle diversity and prevents mode collapse. 
        We also found that overly large step sizes $\alpha \geq 0.1$ destabilize the particle flow.
        
        Although our experiments involve $n_\theta \le 3$ uncertain parameters, the per-iteration cost of SVGD scales linearly in $n_\theta$ and quadratically in $N$, and the principal high-dimensional failure mode (particle collapse under RBF kernels) is mitigated by IMQ's repulsion (Table~\ref{tab:kernel_comparison}) and further alleviated by structured or message-passing Stein updates \cite{zhuo2018messagepassingsteinvariational}, suggesting a viable path toward higher-dimensional parameter spaces.
\section{Conclusions}
    We introduce a Stein variational uncertainty-adaptive model predictive controller for nonlinear systems with latent parameter uncertainty. 
    By constructing a task-dependent posterior and coupling Stein variational inference with control synthesis, we compute task-sensitive robust controllers without the conservatism of worst-case uncertainty design, achieving improved performance and robustness to uncertainty over task-agnostic sampling and overly conservative classical approaches.

\section{Acknowledgements}
    This work is supported by the National Science Foundation under award NSF FRR 2238066. 
    Any opinions, findings, and conclusions or recommendations expressed in this material are those of the authors and do not necessarily reflect the views of the National Science Foundation.


\bibliographystyle{IEEEtran}
\bibliography{main}
\end{document}